\documentclass[11pt]{article}

\usepackage[utf8]{inputenc}
\usepackage[T1]{fontenc}
\usepackage{lmodern}

\usepackage[english]{babel}

\usepackage{amsmath,amssymb,amsthm,mathtools}
\newtheorem{definition}{Definition}[section]
\newtheorem{proposition}[definition]{Proposition}
\theoremstyle{remark}
\newtheorem{remark}{Remark}
\usepackage[a4paper,margin=2.5cm]{geometry}
\usepackage{microtype}

\usepackage{graphicx}
\usepackage{float}

\usepackage{xcolor}

\usepackage{listings}
\usepackage[numbers]{natbib}

\usepackage[
  pdfencoding=auto,
  pdfborder={0 0 0},
  colorlinks=true,
  linkcolor=blue,
  citecolor=blue,
  urlcolor=blue
]{hyperref}

\hypersetup{
  pdftitle={Latent Autoregression via Gaussian-Process Priors in Variational Autoencoders},
  pdfauthor={Yves Ruffenach},
  pdfsubject={Gaussian-process latent autoregression and variational inference},
  pdfkeywords={Gaussian processes, latent autoregression, variational autoencoders, sequential generative models, stochastic processes}
}

\begin{document}

\title{Latent Autoregression via Gaussian-Process Priors\\ in Variational Autoencoders}

\author{
Yves Ruffenach\\
Conservatoire National des Arts et Métiers\\
\texttt{yves.ruffenach.auditeur@lecnam.net}\\
\texttt{yves@ruffenach.net}\\
ORCID: \href{https://orcid.org/0009-0009-4737-0555}{0009-0009-4737-0555}
}

\date{}

\maketitle

\begin{abstract}
We study a class of sequential latent-variable models in which temporal dependence is carried entirely by a Gaussian-process prior on a continuous latent trajectory. Given a finite index set $\{1,\dots,L\}$ and a Gaussian process on $[0,1]$ with covariance kernel $k$, we consider the induced joint Gaussian law of the latent variables $(z_1,\dots,z_L)$ and use its canonical factorization
\[
p(z_{1:L}) \;=\; \prod_{t=1}^L p(z_t \mid z_{<t})
\]
to define a purely latent autoregressive process. In this construction, causality is a probabilistic property of the latent path, obtained by Gaussian conditioning, rather than the result of an explicit recurrence in the observation space. We give an explicit form for the Gaussian conditionals $p(z_t \mid z_{<t})$ and discuss regularity properties of the induced latent process, in particular its covariance structure and its relation to Markov and AR($p$) processes.

We then couple this latent process with a non-autoregressive observation model and derive a variational formulation in the spirit of variational autoencoders: the correlated Gaussian-process prior $p(z_{1:L})$ plays the role of a structured prior on paths, the approximate posterior $q_\phi(z_{1:L}\mid x)$ is amortized by an encoder, and the objective function is a regularized evidence lower bound (ELBO) in which the Kullback-Leibler term measures the deviation from the latent autoregressive prior. We analyse the resulting probabilistic model from the viewpoint of stochastic processes and approximate Bayesian inference, emphasizing the interplay between the Gaussian–process geometry on latent paths and the variational approximation.

As an illustration, we implement this framework with finite-dimensional Gaussian-process priors and a non-autoregressive decoder, and we report numerical results on a standard sequence dataset. These experiments are not the main contribution of the paper: they serve to show that the latent autoregressive scheme induced by the Gaussian-process prior can be trained stably and exploited in practice in a constrained proof-of-concept regime.
\end{abstract}

\bigskip


\section{Introduction}

A central theme in modern probability and statistics is the construction of flexible stochastic processes on high-dimensional spaces that remain amenable to inference. In sequential settings, this often takes the form of latent-variable models in which an unobserved process $(Z_t)_{t=1}^L$ carries the temporal dependence, while observations $(X_t)_{t=1}^L$ are obtained through a (possibly nonlinear) emission mechanism. Classical examples include Markov chains, autoregressive (AR) processes, and state-space models.

In this work we study a different type of latent dynamics, in which temporal structure is imposed by a Gaussian process (GP) prior on a \emph{continuous} latent trajectory. More precisely, we consider a finite collection of latent variables $(z_1,\dots,z_L)$ in $\mathbb{R}^{d_z}$ and assume that they arise as evaluations of an underlying Gaussian process at ordered time points $0 < t_1 < \dots < t_L \le 1$. The induced joint law is multivariate Gaussian, so that
\[
(z_1,\dots,z_L) \sim \mathcal{N}(0, K),
\]
for a covariance matrix $K$ determined by the kernel. From a probabilistic point of view, this construction is classical.

Our contribution is to make explicit and exploit the \emph{causal factorization} of this joint distribution. By writing
\[
p(z_{1:L}) = \prod_{t=1}^L p(z_t \mid z_{<t}),
\]
we obtain a purely latent autoregressive process in which causality is induced by Gaussian conditioning. Unlike Markov models, the conditional law of $z_t$ depends on the entire history $z_{<t}$, and unlike classical autoregressive models on observations, the temporal structure is entirely confined to the latent space. This leads to a family of GP-based latent autoregressive processes governed by the covariance structure of the Gaussian process.

We then couple this latent process with a non-autoregressive observation model, in the spirit of variational autoencoders (VAEs). The Gaussian-process prior plays the role of a correlated prior on latent paths, while an amortized encoder yields an approximate posterior. This gives rise to a variational objective in which the Kullback-Leibler term measures the deviation from the latent autoregressive Gaussian prior. The resulting model can be viewed as a probabilistic alternative to autoregressive sequence models, where temporal dependence is expressed analytically at the latent level and the decoder is fully parallel.

The main contributions of this paper are as follows:
\begin{itemize}
  \item we formalize a class of latent autoregressive processes obtained by causal factorization of a Gaussian-process prior on a finite grid;
  \item we derive explicit Gaussian conditional distributions $p(z_t \mid z_{<t})$ and discuss basic properties (existence, covariance structure, relation to Markov and AR($p$) processes);
  \item we embed this latent process in a variational framework and obtain an ELBO-type objective in which the KL term is computed against the correlated prior;
  \item we illustrate, on a standard sequence dataset, that this construction can be implemented in practice and trained stably within a constrained proof-of-concept setup.
\end{itemize}

\section{Related Work}

\subsection{Variational Inference and Latent Variable Models}

Variational methods have become a standard tool for approximate Bayesian
inference in latent-variable models. In its simplest form, a variational
approximation replaces an intractable posterior distribution
$p_{\theta}(z \mid x)$ by a tractable distribution $q_{\phi}(z \mid x)$
minimizing the Kullback-Leibler divergence. The resulting objective,
often referred to as the evidence lower bound (ELBO), reads
\[
\mathcal{L}(x)
=
\mathbb{E}_{q_{\phi}(z \mid x)}[\log p_{\theta}(x \mid z)]
-
D_{\mathrm{KL}}\!\left(q_{\phi}(z \mid x)\,\|\, p_{\theta}(z)\right),
\]
and provides a tractable surrogate for the marginal likelihood
$\log p_{\theta}(x)$.

Within this general framework, variational autoencoders (VAEs)
\cite{Kingma2014VAE} constitute a class of latent-variable models in
which the approximate posterior $q_{\phi}$ and the likelihood model
$p_{\theta}$ are parameterized by neural networks. Many structural
variants exist: scaled Kullback-Leibler regularization
($\beta$-VAE~\cite{Higgins2017BetaVAE}), hierarchical priors
\cite{Sonderby2016LadderVAE}, and sequential latent structures
\cite{Fraccaro2016Stochastic,Karl2017DVBF}. From a probabilistic
perspective, these models can be viewed as flexible non-linear
state-space models, in which variational inference plays the role of an
amortized filtering or smoothing procedure.

A key theme in this literature is the construction of latent spaces that
retain meaningful global structure beyond what is directly observable.
This aspect is central to few-shot and one-shot generalization
\cite{Rezende2016OneShot}, where latent regularity plays the same role
as prior smoothness in classical Bayesian non-parametrics.

\subsection{Gaussian Processes and Correlated Latent Priors}

Gaussian processes (GPs) \cite{Rasmussen2006GPML} provide a canonical
class of stochastic processes indexed by continuous time or space, and
form the basis for a long tradition in Bayesian non-parametrics.
Given a mean function $m$ and a covariance kernel $k$, a GP satisfies
\[
f(t) \sim \mathcal{GP}\!\bigl(m(t), k(t,t')\bigr),
\]
and any finite collection $(f(t_1),\dots,f(t_L))$ is jointly Gaussian
with covariance matrix $K=(k(t_i,t_j))_{i,j}$.

When GPs are used as priors for latent trajectories, their covariance
structure can regularize or constrain the evolution of the latent
variables. This idea has been studied under various forms in the machine
learning literature. In GP-VAEs
\cite{Casale2018GPPVAE,Fortuin2020GPVAE,Pearce2020BAGP}, the latent
variables of a VAE are endowed with a correlated Gaussian prior
$z \sim \mathcal{N}(0, K)$, which induces a form of temporal or spatial
continuity. Such models produce smooth latent paths rather than
independent latent points, but the induced correlations are symmetric
and do not encode causal direction. Consequently, these approaches yield
latent continuity but not latent autoregression in the sense of a
factorization
$p(z_{1:L}) = \prod_{t=1}^L p(z_t \mid z_{<t})$.

\subsection{Latent Dependence, Conditioning, and One-Shot Behavior}

Several works have emphasized that latent correlations can substantially
improve generalization when data are scarce. Rezende et
al.~\cite{Rezende2016OneShot} demonstrated that correlated latent priors
combined with variational inference can produce coherent samples from
single examples. More recent work, such as dense GP layers
\cite{Wang2021DenseGP}, shows that embedding GP-like structures into
deep networks enhances contextual coherence.

The ability of Gaussian processes to provide analytically tractable
conditional distributions plays a central role in these constructions.
For a GP evaluated at $\{t_1,\dots,t_L\}$, the conditional distribution
of $z_t$ given $z_{<t}$ is Gaussian:
\[
p(z_t \mid z_{<t})
=
\mathcal{N}\!\left(
k_{12}^\top K_{11}^{-1} z_{<t},
K_{22}-k_{12}^\top K_{11}^{-1} k_{12}
\right),
\]
where $K_{11}$ is the covariance of $z_{<t}$ and $k_{12}$ their
cross-covariance with $z_t$. This conditional representation is a
classical result of multivariate Gaussian theory, and forms the
probabilistic foundation of the latent autoregressive model developed in
the present work.

\subsection{Sequential Latent Models and the Lack of Full Causality}

Sequential latent-variable models are well documented. Stochastic
recurrent VAEs \cite{Fraccaro2016Stochastic}, deep Kalman-like models
\cite{Karl2017DVBF}, and continuous-time latent SDE/SSM models
\cite{Zhou2023DeepLatentSSM,Klushyn2021LatentMatters} provide
increasingly expressive families of latent processes, typically of the
form
\[
p(z_{1:L})=\prod_{t=1}^L p(z_t \mid z_{t-1}),
\]
or more elaborate Markovian structures. However, these models rely on
parametric state transitions whose temporal dependence is learned
through neural weights. They do not exploit the analytic conditional
structure of Gaussian processes, nor do they provide a full causal
factorization with dependence on the entire latent history.

Inverse autoregressive flows \cite{Kingma2016IAF} manipulate latent
densities using autoregressive maps, but the induced dependence is
parametric and does not correspond to a causal stochastic dynamic.
Similarly, existing GP-based approaches impose correlations but not
direction.

To our knowledge, no prior work combines:
\begin{itemize}
    \item a correlated Gaussian-process prior on a latent trajectory,
    \item a fully causal factorization $p(z_{1:L})=\prod_{t=1}^L p(z_t\mid z_{<t})$ obtained analytically from GP conditioning,
    \item and a non-autoregressive likelihood model.
\end{itemize}
This combination produces a latent autoregressive process in which
causal structure arises from probabilistic geometry rather than learned
transition weights.

\subsection{Computational Considerations}

Inference with Gaussian processes is limited by the cost of inverting
covariance matrices, nominally $O(L^3)$. Several scalable numerical
methods have been proposed, such as inducing-point approximations
\cite{Titsias2009InducingPoints} and matrix-free linear solvers
implemented via structured or stochastic approximations. In practice, we
make use of the BBMM method of \cite{Gardner2018GPyTorch}, which permits
quasi-quadratic GP inference in $O(L^2)$ operations and enables
joint optimization of kernel hyperparameters with variational objectives.

The kernel choice (RBF, Matérn \cite{Guttorp2006Matern}, or spectral
mixture \cite{Wilson2013SpectralMixture}) determines the regularity of
the latent process, in the classical sense of sample-path smoothness,
and influences the conditional means and variances appearing in the
latent autoregressive factorization.

\subsection{Summary}

The present work lies at the intersection of three bodies of literature:
\begin{itemize}
    \item variational inference and latent-variable models,
    \item Gaussian-process priors on latent trajectories,
    \item sequential latent models and nonparametric stochastic processes.
\end{itemize}
While prior approaches have used Gaussian processes to induce latent
correlations, they have not exploited the analytic causal factorization
available for finite GP evaluations. The construction developed here
extends this direction by placing temporal causality in the latent
space, yielding a latent autoregressive process governed entirely by the
covariance structure of the Gaussian process. We view this as a
probabilistic alternative to learned state transitions, and as a
conceptual step toward generative models in which sequential structure
is expressed analytically rather than parametrically.

\section{Mathematical Preliminaries}
\label{sec:math-preliminaries}

We briefly recall the basic Gaussian-process (GP) and matrix-analytic
facts used throughout the paper. The goal is to make the probabilistic
assumptions and the linear-algebraic structure fully explicit.

In the theoretical developments we use L for the sequence length, while in the complexity and experimental sections we write T. In the discrete setting considered here, these play the same role.

\subsection{Gaussian Processes on a Finite Grid}

Let $(\Omega,\mathcal{F},\mathbb{P})$ be a probability space and
consider a centred Gaussian process
\[
f \sim \mathcal{GP}(0,k)
\]
indexed by $[0,1]$, where $k : [0,1]\times[0,1] \to \mathbb{R}$ is a
positive-definite kernel.

\begin{itemize}
  \item[(H1)] \textbf{Kernel regularity.}
  The kernel $k$ is continuous on $[0,1]^2$ and positive definite in
  the sense that, for any finite family of pairwise distinct points
  $0<t_1<\dots<t_L\le 1$, the Gram matrix
  \[
  K_{tt} \;=\; \bigl(k(t_i,t_j)\bigr)_{1\le i,j\le L}
  \]
  is symmetric positive definite.

  \item[(H2)] \textbf{Latent dimensionality.}
  The latent dimension $d_z \in \mathbb{N}$ is fixed, while the
  sequence length $L$ may vary. We write
  \[
  z_t \in \mathbb{R}^{d_z}, \qquad
  z_{1:L} = (z_1,\dots,z_L) \in \mathbb{R}^{L d_z}.
  \]
\end{itemize}

Under (H1), for any choice of $0<t_1<\dots<t_L\le 1$ we can define the
finite-dimensional Gaussian vector
\[
z_{1:L}
=
\bigl(f(t_1),\dots,f(t_L)\bigr),
\]
which is distributed according to
\[
z_{1:L}
\sim
\mathcal{N}\bigl(0, K_{tt}\bigr).
\]

In the latent-variable model considered in this work, we use a
$d_z$-dimensional GP, implemented as $d_z$ i.i.d.\ copies of $f$.
This yields the block-structured covariance
\[
\mathrm{Cov}(z_{1:L})
=
K_{tt} \otimes I_{d_z},
\]
where $I_{d_z}$ denotes the $d_z\times d_z$ identity matrix and
$\otimes$ is the Kronecker product. The Kronecker structure will be
used explicitly in the computation of KL divergences and log-densities.

\subsection{Gaussian Conditioning}

We recall the classical conditioning formulas for multivariate
Gaussians in the notation used later for the causal factorization.

Let $(Y_1,Y_2)$ be a centred Gaussian vector with
\[
\begin{pmatrix} Y_1 \\[1ex] Y_2 \end{pmatrix}
\sim
\mathcal{N}\!\left(
0,\;
\begin{bmatrix}
\Sigma_{11} & \Sigma_{12} \\
\Sigma_{21} & \Sigma_{22}
\end{bmatrix}
\right),
\]
where $\Sigma_{11}$ is invertible. Then the conditional distribution of
$Y_2$ given $Y_1=y_1$ is Gaussian with
\begin{align}
\mathbb{E}\bigl[Y_2 \mid Y_1=y_1\bigr]
  &= \Sigma_{21}\Sigma_{11}^{-1} y_1,
  \label{eq:gaussian-cond-mean}
  \\
\mathrm{Cov}\bigl(Y_2 \mid Y_1=y_1\bigr)
  &= \Sigma_{22} - \Sigma_{21}\Sigma_{11}^{-1}\Sigma_{12}.
  \label{eq:gaussian-cond-cov}
\end{align}

Applied to the GP evaluation vector $z_{1:L}$, this yields the
sequential conditioning formulas used to construct the latent
autoregressive factorization. More precisely, for each $t\in\{2,\dots,L\}$
we partition
\[
z_{1:L}
=
(z_{<t}, z_t),
\qquad
z_{<t} = (z_1,\dots,z_{t-1}),
\]
and the covariance matrix accordingly as
\[
\Sigma_{11} = K_{(<t,<t)}\otimes I_{d_z},\quad
\Sigma_{12} = k_{(<t,t)}\otimes I_{d_z},\quad
\Sigma_{22} = k_{(t,t)}\otimes I_{d_z},
\]
where $K_{(<t,<t)}$ is the Gram matrix restricted to indices
$\{1,\dots,t-1\}$ and $k_{(<t,t)}$ is the column vector of
cross-covariances between $t$ and $\{1,\dots,t-1\}$.

Using \eqref{eq:gaussian-cond-mean}–\eqref{eq:gaussian-cond-cov} we
obtain conditional means $\mu_t$ and covariances $\Sigma_t$ as
\[
p(z_t \mid z_{<t})
=
\mathcal{N}\bigl(\mu_t(z_{<t}),\Sigma_t\bigr),
\]
with explicit formulas given in Section~\ref{sec:latent-ar-process}.

\paragraph{Notation.}
We write $k_{(t,<t)} \in \mathbb{R}^{t-1}$ for the vector of
cross-covariances between $t$ and $\{1,\dots,t-1\}$.  
The alternative notation $k_{(<t,t)}$ refers to the same object (transpose
conventions only).

\subsection{Kronecker Products and Tensor Structures}

We recall the basic identities used throughout the paper for Kronecker
products. For matrices $A \in \mathbb{R}^{m\times m}$ and
$B \in \mathbb{R}^{n\times n}$,
\[
A \otimes B \in \mathbb{R}^{mn\times mn}
\]
is defined blockwise by
\[
A \otimes B
=
\begin{bmatrix}
a_{11} B & \cdots & a_{1m} B\\
\vdots   & \ddots & \vdots  \\
a_{m1} B & \cdots & a_{mm} B
\end{bmatrix}.
\]

In our setting:
\begin{itemize}
  \item $K_{tt} \in \mathbb{R}^{L\times L}$ encodes temporal correlations;
  \item $I_{d_z} \in \mathbb{R}^{d_z\times d_z}$ acts on the latent
        coordinates;
  \item the joint covariance of $z_{1:L} \in \mathbb{R}^{Ld_z}$ is
        $K_{tt} \otimes I_{d_z}$.
\end{itemize}

We use the standard properties:
\begin{align*}
\det(K_{tt} \otimes I_{d_z})
  &= \det(K_{tt})^{d_z},\\
(K_{tt} \otimes I_{d_z})^{-1}
  &= K_{tt}^{-1} \otimes I_{d_z},
\end{align*}
whenever $K_{tt}$ is invertible.
These identities are used implicitly in the computation of Gaussian
log-densities and KL divergences between diagonal posteriors and the
GP prior.

\paragraph{Standing assumptions.}
Assumptions (H1)–(H4) will be referred to throughout Sections 3 and 4.
Unless otherwise stated, all probabilistic statements are made under
hypotheses (H1)--(H3) of
Sections~\ref{sec:math-preliminaries} and~\ref{subsec:complexity-gp}.
In particular:
\begin{itemize}
  \item the kernel $k$ is continuous and positive definite on $[0,1]^2$;
  \item the latent dimension $d_z$ is fixed while the sequence length
        $L$ may vary (within the regime $L\le T_{\max}$ used in practice);
  \item all Gram matrices $K_{tt}$ arising from such sequences are
        numerically well conditioned, with a uniform lower bound on
        their smallest eigenvalue, enforced by a jitter term
        $\varepsilon I$.
\end{itemize}
Assumption (H4) is purely numerical and is only invoked when discussing
the BBMM complexity estimates in Section~\ref{subsec:complexity-gp}.
Under these conditions, the causal Gaussian factorization in
Proposition~\ref{prop:causal-factorization} is well defined and the
complexity considerations of Section~\ref{subsec:complexity-gp} apply.

\section{Probabilistic Framework and Methodology}

\subsection{Complexity of Gaussian-Process Latent Dynamics}
\label{subsec:complexity-gp}

We briefly summarize the computational complexity associated with the
latent Gaussian-process prior and state the numerical assumptions
under which the proposed scheme is implemented.

\paragraph{Exact complexity.}
For a sequence of length $T$ and a kernel $k$ evaluated on
$0<t_1<\dots<t_T\le 1$, the temporal covariance is
$K_{tt} \in \mathbb{R}^{T\times T}$. Exact GP computations (log-density
and sampling) require a Cholesky factorization
\[
K_{tt} = L L^\top,
\]
with cost $O(T^3)$ operations and $O(T^2)$ memory. In the latent model
considered here, the latent dimension $d_z$ is fixed and the covariance
of $z_{1:T}$ is $K_{tt}\otimes I_{d_z}$, so that the dominant cost still
comes from the $T\times T$ part; Kronecker structure removes only the
dependence on $d_z$ in the factorization.

\paragraph{Approximate complexity.}
To avoid the cubic dependence on $T$, we rely on matrix-free GP
inference, in particular BBMM-style methods based on conjugate
gradients (CG) and stochastic trace estimation
\cite{Gardner2018GPyTorch}. In this regime:

\begin{itemize}
  \item the cost of a single CG iteration for solving
        $K_{tt} v = b$ is $O(T^2 d_z)$, assuming dense kernels and no
        additional structure;
  \item the total cost is therefore
        $O\bigl(T^2 d_z\, n_{\mathrm{iter}}\bigr)$, where
        $n_{\mathrm{iter}}$ is the number of CG iterations required to
        reach a prescribed tolerance $\varepsilon_{\mathrm{CG}}$;
\item classical CG theory implies that, for fixed
$\varepsilon_{\mathrm{CG}}$, one typically has
\[
n_{\mathrm{iter}}
= O\!\bigl(\sqrt{\kappa(K_{tt})}\,
\log(1/\varepsilon_{\mathrm{CG}})\bigr),
\]
where $\kappa(K_{tt})$ denotes the condition number of $K_{tt}$.

\end{itemize}

Under the spectral bound (H3) and a fixed tolerance
$\varepsilon_{\mathrm{CG}}$, the quantity $n_{\mathrm{iter}}$ remains
bounded over the finite set of covariance matrices considered, so that
the effective scaling observed in our experiments is close to
$O(T^2 d_z)$.

\paragraph{Numerical assumptions.}
The next two assumptions are purely numerical: they are only used to
justify the complexity claims for BBMM-style GP inference and play no
role in the measure-theoretic existence and uniqueness results of
Section~\ref{sec:math-preliminaries}.

\begin{itemize}
  \item[(H3)] \textbf{Uniform spectral bounds on the experimental grid.}
  There exists $T_{\max}\in\mathbb{N}$ and constants
  $0 < \underline{\Lambda} \le \overline{\Lambda} < \infty$ such that,
  for all sequence lengths $T\le T_{\max}$ considered in our experiments
  and all kernel hyperparameters of interest, the eigenvalues of $K_{tt}$
  satisfy
  \[
  \underline{\Lambda} \le \Lambda_{\min}(K_{tt})
  \le \Lambda_{\max}(K_{tt}) \le \overline{\Lambda}.
  \]
  In practice, we add a small jitter term $\varepsilon I$ to $K_{tt}$,
  which enforces $\Lambda_{\min}(K_{tt}) \ge \varepsilon > 0$ on the
  finite set of covariance matrices actually used.

  \item[(H4)] \textbf{Controlled CG tolerance.}
  CG iterations in BBMM are stopped when the residual norm falls below
  a fixed tolerance $\varepsilon_{\mathrm{CG}}>0$. We assume that this
  tolerance is chosen such that the resulting approximation of
  $K_{tt}^{-1}$ is sufficiently accurate for the purposes of computing
  KL divergences and ELBO estimates; no probabilistic statement in this
  paper depends on the exact value of $\varepsilon_{\mathrm{CG}}$.
\end{itemize}

\paragraph{Latent vs.\ observation-level complexity.}
The key reduction in complexity comes from the fact that temporal
dependence is modeled on a latent sequence $z_{1:L}$ with $L\ll N$ and
$d_z\ll d$, rather than on the full observation sequence $x_{1:N}$.
Under the assumptions above, the GP-related cost scales as
$O(L^2 d_z)$, to be compared with $O(N^2 d)$ for a typical
self-attention layer on the observation tokens.

In other words, the computational gain is not purely due to numerical
approximations of GP inference, but also to a change in representation:
temporal dependence is encoded once in the covariance structure of the
latent process, rather than recomputed layer by layer in the observation
space.

\subsection{Latent Architecture and Observation Model}

The model architecture is organized around three components: an
amortized encoder, a correlated latent process, and an observation
model. This separation is typical of variational formulations and is
particularly important here, as it isolates the temporal dynamics within
the latent space.

\paragraph{Encoder.}
The encoder maps an observed sequence $x_{1:N}$ to a family of
approximate posteriors over latent trajectories. We write
\[
q_{\phi}(z_{1:L} \mid x_{1:N}),
\]
and in the present work take $q_{\phi}$ to be Gaussian, with means and
covariances parameterized by a temporal network (for instance a causal
convolutional architecture). The precise parametric form of the encoder
is not essential for the probabilistic formulation; its role is to
provide an amortized approximation to the intractable posterior
$p_{\theta}(z_{1:L} \mid x_{1:N})$.

\paragraph{Latent process.}
The core of the model is a Gaussian-process prior on the latent
sequence $z_{1:L}$ endowed with a causal factorization. Formally, we
assume that $(z_1,\dots,z_L)$ arises from a GP evaluated at ordered time
points, and we use the induced Gaussian conditionals to define
\[
p_{\theta}(z_{1:L})
=
\prod_{t=1}^{L} p_{\theta}(z_t \mid z_{<t}),
\qquad
p_{\theta}(z_t \mid z_{<t}) = \mathcal{N}(m_t, \Sigma_t),
\]
where $(m_t,\Sigma_t)$ are the predictive mean and covariance obtained
from the Gaussian conditioning formulas. The covariance kernel
$k_{\psi}(t,t')$ determines the structure of dependence: for instance, a
squared-exponential or Matérn kernel \cite{Guttorp2006Matern} enforces
temporal continuity, while spectral kernels
\cite{Wilson2013SpectralMixture} can encode approximate periodicities.

\paragraph{Observation model (decoder).}
Given a latent trajectory $z_{1:L}$, observations are generated by an
observation model
\[
p_{\theta}(x_{1:N} \mid z_{1:L})
=
\prod_{n=1}^{N} p_{\theta}(x_n \mid z_{1:L}),
\]
which is conditionally independent across positions given the entire
latent sequence. This structure permits fully parallel generation in the
observation space: all $x_n$ are sampled simultaneously once $z_{1:L}$
is known.

The functional form of $p_{\theta}(x_n \mid z_{1:L})$ can be chosen to
match the nature of the data: Gaussian likelihoods for continuous
signals, categorical likelihoods for discrete symbols, etc. From the
probabilistic standpoint, the decoder is simply a family of conditional
distributions indexed by the latent path; its parametrization (via
neural networks or otherwise) is orthogonal to the latent autoregressive
construction.

This separation highlights an important conceptual distinction between
classical autoregressive models and the present latent-variable
approach. In a standard autoregressive model one specifies
\[
p(x_{1:N}) = \prod_{t=1}^{N} p(x_t \mid x_{<t}),
\]
whereas in the GP-based latent framework one works with
\[
p(x_{1:N}) = \int p_{\theta}(x_{1:N} \mid z_{1:L})\, p_{\theta}(z_{1:L}) \,\mathrm{d}z_{1:L},
\]
so that temporal dependence is encoded in $p_{\theta}(z_{1:L})$ rather
than directly in the observation conditionals.

\subsection{Purely Latent Autoregression}
\label{sec:latent-ar-process}

\begin{proposition}[Causal Gaussian Factorization]
\label{prop:causal-factorization}
Assume \textnormal{(H1)}-\textnormal{(H2)} hold and let
$0<t_1<\dots<t_L\le 1$ be fixed. Consider the centred Gaussian vector
\[
z_{1:L}
\sim
\mathcal{N}\bigl(0, K_{tt} \otimes I_{d_z}\bigr).
\]
Then:

\begin{enumerate}
  \item For each $t\in\{1,\dots,L\}$ there exists a unique Gaussian
  conditional distribution
  \[
  p(z_t \mid z_{<t})
  =
  \mathcal{N}\bigl(\mu_t(z_{<t}),\Sigma_t\bigr),
  \]
  where $\mu_t$ is affine in $z_{<t}$ and $\Sigma_t$ is a symmetric
  positive-definite matrix in $\mathbb{R}^{d_z\times d_z}$.

  \item The joint law admits the unique causal factorization
  \begin{equation}
  p(z_{1:L})
  =
  \prod_{t=1}^{L} p(z_t \mid z_{<t}),
  \label{eq:gp-causal-factorization}
  \end{equation}
  where, for $t\ge 2$, the mean and covariance are given by the
  Gaussian-conditioning formulas
  \begin{align}
  \mu_t(z_{<t})
    &= \bigl(k_{(t,<t)}^\top K_{(<t,<t)}^{-1}\bigr) \otimes I_{d_z}\; z_{<t},
    \label{eq:mu-t}
    \\
  \Sigma_t
    &= \bigl(k_{(t,t)} - k_{(t,<t)}^\top K_{(<t,<t)}^{-1}k_{(t,<t)}\bigr)
       \, I_{d_z},
    \label{eq:sigma-t}
  \end{align}
  with the convention that $p(z_1)=\mathcal{N}(0,k(t_1,t_1)I_{d_z})$.
\end{enumerate}
\end{proposition}

\begin{proof}
By (H1), $K_{tt}$ is symmetric positive definite, hence invertible, so
the finite-dimensional Gaussian measure
$\mathcal{N}(0,K_{tt}\otimes I_{d_z})$ on $\mathbb{R}^{L d_z}$ is
non-degenerate. In particular, for each $t\in\{2,\dots,L\}$ the pair
$(z_{<t},z_t)$ is a centred Gaussian vector taking values in the
Polish space $\mathbb{R}^{(t-1)d_z}\times\mathbb{R}^{d_z}$.

The existence and uniqueness (up to $p(z_{<t})$-null sets) of regular
conditional distributions $p(z_t\mid z_{<t})$ then follows from the
general disintegration theorem for probability measures on Polish
spaces; see for instance Kallenberg~\cite[Theorem~6.3]{Kallenberg2002}.
In the Gaussian case, these regular conditional laws are themselves
Gaussian and are given explicitly by the conditioning formulas
\eqref{eq:gaussian-cond-mean}--\eqref{eq:gaussian-cond-cov}.

For each $t\ge 2$ we partition
\[
z_{1:L}=(z_{<t},z_t),\qquad
z_{<t}\in\mathbb{R}^{(t-1)d_z},\;z_t\in\mathbb{R}^{d_z},
\]
and the covariance as
\[
\Sigma_{11} = K_{(<t,<t)}\otimes I_{d_z},\quad
\Sigma_{12} = k_{(<t,t)}\otimes I_{d_z},\quad
\Sigma_{22} = k_{(t,t)}\otimes I_{d_z},
\]
where $K_{(<t,<t)}\in\mathbb{R}^{(t-1)\times(t-1)}$ and
$k_{(<t,t)}\in\mathbb{R}^{t-1}$. Hence
$\Sigma_{11}\in\mathbb{R}^{(t-1)d_z\times (t-1)d_z}$,
$\Sigma_{12}\in\mathbb{R}^{(t-1)d_z\times d_z}$ and
$\Sigma_{22}\in\mathbb{R}^{d_z\times d_z}$, so that the Kronecker
expressions in \eqref{eq:mu-t}--\eqref{eq:sigma-t} are dimensionally
consistent. Applying
\eqref{eq:gaussian-cond-mean}--\eqref{eq:gaussian-cond-cov} to this
partition yields the claimed formulas for $\mu_t$ and $\Sigma_t$.

Finally, since $p(z_{1:L})$ admits a Lebesgue density on
$\mathbb{R}^{L d_z}$, the factorization
\eqref{eq:gp-causal-factorization} follows from the chain rule for
densities, and its uniqueness from the uniqueness (up to null sets) of
the conditional Gaussian laws $p(z_t\mid z_{<t})$.
\end{proof}

We now formalize the notion of purely latent autoregression alluded to
in the introduction. The setting is that of a latent process
$(z_1,\dots,z_L)$, with $z_t \in \mathbb{R}^{d_z}$, endowed with a
joint distribution that factorizes causally in latent space.

\subsubsection{General principle}

A model exhibits purely latent autoregression if the joint law of
$z_{1:L}$ admits a factorization
\begin{equation}
p(z_{1:L}) = \prod_{t=1}^{L} p(z_t \mid z_{<t}),
\label{eq:latent-ar-factorization}
\end{equation}
where each conditional $p(z_t \mid z_{<t})$ depends only on the past
latents and not on the observations. In the Gaussian-process-based
construction considered here, these conditionals are Gaussian:
\begin{align}
p(z_t \mid z_{<t}) &= \mathcal{N}(\mu_t, \Sigma_t), \\
\mu_t(z_{<t})
  &= \bigl(k_{(t,<t)}^{\top} K_{(<t,<t)}^{-1}\bigr) \otimes I_{d_z}\; z_{<t}, \\
\Sigma_t
  &= \bigl(k_{(t,t)} - k_{(t,<t)}^{\top} K_{(<t,<t)}^{-1} k_{(t,<t)}\bigr) I_{d_z}.
\end{align}

where $K_{(<t,<t)}$ is the covariance of $(z_1,\dots,z_{t-1})$ and
$k_{(t,<t)}$ their cross-covariance with $z_t$. These formulas follow
from classical multivariate Gaussian conditioning and show that the
latent dynamics are governed by the covariance structure of the GP
rather than by parametric transition weights.

In this sense the model can be interpreted as a Bayesian autoregression:
memory of the past is transmitted through correlations in the prior
rather than through explicitly learned recurrence. A trajectory
$z_{1:L}$ is first sampled according to \eqref{eq:latent-ar-factorization},
and observations are then generated conditionally on this latent path.

\subsubsection{Comparison with Markov and observation-level autoregression}

The factorization \eqref{eq:latent-ar-factorization} should be
contrasted with the first-order Markov property
\[
p(z_{1:L}) = \prod_{t=1}^{L} p(z_t \mid z_{t-1}),
\]
in which the conditional law at time $t$ depends only on the immediate
predecessor $z_{t-1}$. Markov models are often sufficient for physical
processes with short-range dependencies, but they are limited in their
ability to express long-range temporal structure without additional
hierarchy.

On the other hand, classical autoregressive models at the observation
level specify
\[
p(x_{1:N}) = \prod_{t=1}^{N} p(x_t \mid x_{<t}),
\]
and therefore place causality directly on the observables. This
formulation permits rich dependence but typically entails sequential
generation and sensitivity to local errors, as each $x_t$ conditions on
all preceding observations.

The present model occupies an intermediate position. It retains full
latent causality (each $z_t$ depends on $z_{<t}$) but operates in a
continuous latent space where the covariance structure enforces smooth,
probabilistic coherence. Temporal dependence is thus decoupled from the
symbolic or observed sequence, which is recovered in a second step via
the observation model.

\subsubsection{Formal definition}

For future reference, we state the definition adopted in this paper.

\begin{definition}[Latent autoregressive process]
\label{def:latent-ar}
Let $(z_1,\dots,z_L)$ be random variables with values in
$\mathbb{R}^{d_z}$, defined on a probability space
$(\Omega,\mathcal{F},\mathbb{P})$. We say that $(z_1,\dots,z_L)$ forms
a \emph{latent autoregressive process} if:

\begin{enumerate}
  \item The joint law of $z_{1:L}$ admits a density $p(z_{1:L})$ with
        respect to the Lebesgue measure on $\mathbb{R}^{L d_z}$.

  \item There exist Borel-measurable functions
  $m_t : (\mathbb{R}^{d_z})^{t-1} \to \mathbb{R}^{d_z}$ and symmetric
  positive-definite matrices $\Sigma_t \in \mathbb{R}^{d_z\times d_z}$
  such that
  \begin{equation}
  p(z_{1:L})
  =
  p(z_1)\,
  \prod_{t=2}^{L}
  \varphi\bigl(z_t; m_t(z_{<t}), \Sigma_t\bigr),
  \label{eq:def-latent-ar}
  \end{equation}
  where $\varphi(\cdot;m,\Sigma)$ denotes the Gaussian density with mean
  $m$ and covariance $\Sigma$.
\end{enumerate}

In the GP-based model developed in this paper, $m_t$ and $\Sigma_t$ are
not free parameters: they are given uniquely by the Gaussian conditioning
formulas \eqref{eq:mu-t}–\eqref{eq:sigma-t} associated with the kernel
$k$ and the evaluation times $(t_1,\dots,t_L)$.
\end{definition}

\begin{remark}[On densities vs.\ regular conditional laws]
\label{rem:density-vs-rcd}
In Definition~\ref{def:latent-ar} we formulate latent autoregression at
the level of Lebesgue densities. In the GP-based construction of
Section~\ref{sec:latent-ar-process}, hypothesis~(H1) implies that
$K_{tt}\otimes I_{d_z}$ is strictly positive definite, so the joint law
of $z_{1:L}$ is a non-degenerate Gaussian measure on $\mathbb{R}^{L d_z}$
and therefore admits such a density.

If one wishes to include degenerate Gaussian priors (for instance in the
limit of a vanishing nugget), the same notion can be reformulated in
terms of regular conditional probabilities
$\mathbb{P}(z_t\in\cdot \mid z_{<t})$ without reference to densities,
using standard disintegration results on Polish spaces; see again
Kallenberg~\cite[Theorem~6.3]{Kallenberg2002}.
\end{remark}

In the GP-based model developed here, the conditional means $m_t$ and the
covariances $\Sigma_t$ are determined analytically by the covariance kernel and the conditioning
formulas. When $f_{\theta}$ depends only on $z_{t-1}$ one recovers a
Markov structure; when it depends on the full history $z_{<t}$ one
obtains full latent autoregression. This definition emphasizes that the
causal structure is entirely internal to the latent process and distinct
from any autoregression at the observation level.

\subsection{Variational Formulation and Objective}

We now recall the variational formulation used for inference and learning
in the model. Let $x_{1:N}$ denote an observed sequence and $z_{1:L}$ a
latent trajectory. The joint distribution factorizes as
\begin{equation}
p_{\theta}(x_{1:N}, z_{1:L})
=
p_{\theta}(x_{1:N} \mid z_{1:L})\, p_{\theta}(z_{1:L}),
\end{equation}
where $p_{\theta}(z_{1:L})$ is the latent autoregressive GP prior
described above, and $p_{\theta}(x_{1:N} \mid z_{1:L})$ the observation
model.

The marginal likelihood of the data is
\begin{equation}
\log p_{\theta}(x_{1:N})
=
\log \int p_{\theta}(x_{1:N}, z_{1:L})\, \mathrm{d}z_{1:L},
\end{equation}
which is typically intractable. A standard variational approximation
introduces an auxiliary posterior $q_{\phi}(z_{1:L} \mid x_{1:N})$ and
uses Jensen's inequality to obtain the evidence lower bound
\begin{equation}
\log p_{\theta}(x_{1:N})
\;\geq\;
\mathbb{E}_{q_{\phi}(z_{1:L} \mid x_{1:N})}
\bigl[\log p_{\theta}(x_{1:N} \mid z_{1:L})\bigr]
-
D_{\mathrm{KL}}\!\bigl(q_{\phi}(z_{1:L} \mid x_{1:N}) \,\Vert\, p_{\theta}(z_{1:L})\bigr).
\end{equation}

The first term encourages $p_{\theta}(x_{1:N} \mid z_{1:L})$ to place
mass on the observations when $z_{1:L}$ is sampled from the variational
posterior; the second term regularizes the variational posterior toward
the correlated GP prior. The variational parameters $\phi$ and the model
parameters $\theta$ are learned by maximizing this lower bound over
data.

In practice, it is often convenient to introduce a scale parameter
$\beta>0$ on the divergence term:
\begin{equation}
\mathcal{L}_{\beta}(\theta,\phi)
=
\mathbb{E}_{q_{\phi}(z_{1:L} \mid x_{1:N})}
\bigl[\log p_{\theta}(x_{1:N} \mid z_{1:L})\bigr]
-
\beta\,
D_{\mathrm{KL}}\!\bigl(q_{\phi}(z_{1:L} \mid x_{1:N}) \,\Vert\, p_{\theta}(z_{1:L})\bigr),
\end{equation}
which controls the trade-off between fidelity to the observations and
adherence to the latent autoregressive prior. The case $\beta=1$
corresponds to the usual ELBO; values $\beta \neq 1$ interpolate between
stronger regularization and more flexible reconstructions.

\subsubsection{Notation and assumptions}

For clarity, we summarize the main notation:
\begin{itemize}
    \item $x_{1:N}$: observed sequence;
    \item $z_{1:L}$: latent sequence, with $L$ not necessarily equal to $N$;
    \item $p_{\theta}(z_{1:L})$: latent autoregressive GP prior;
    \item $q_{\phi}(z_{1:L} \mid x_{1:N})$: variational posterior (encoder);
    \item $p_{\theta}(x_{1:N} \mid z_{1:L})$: observation model (decoder);
    \item $k_{\psi}$: covariance kernel, parameterized by hyperparameters $\psi$;
    \item $\theta,\phi,\psi$: collections of parameters for the prior, variational family, and kernel.
\end{itemize}

Under these assumptions, the model defines a fully differentiable
generative system in which temporal dependence is embedded in the latent
prior, and the interaction with data is mediated by the observation
model and the variational posterior.

\section{Experiments - Empirical Validation and Limitations of the Latent Autoregressive Scheme}

This chapter provides an empirical validation of the purely latent autoregressive scheme described in the previous sections.  
The aim is not to claim any form of asymptotic or large-scale optimality, but rather to demonstrate, in a controlled and reproducible regime, that the proposed probabilistic construction behaves as expected.

More precisely, we seek to verify the following points:
\begin{itemize}
    \item a GP-VAE endowed with fully latent causality is trainable and numerically stable;
    \item the correlated latent prior is effectively used, in a way that cannot be reduced to a mere \emph{KL-capping} artefact;
    \item the two sampling schemes (sequential vs.\ parallel) are empirically consistent with the same joint law;
    \item in the constrained regime considered here, the model surpasses a minimal autoregressive baseline in the observation space.
\end{itemize}

Throughout this chapter, all conclusions should be read as a \emph{local} validation of the latent-autoregressive scheme under a restricted computational budget, rather than as a large-scale benchmark.

\subsection{Experimental Objectives and Tested Hypotheses}

We consider a corpus $\mathcal{D} = \{x^{(i)}_{1:N_i}\}_{i=1}^{M}$ and models of the form
\[
p_{\theta}(x, z) = p_{\theta}(x \mid z)\, p_{\theta}(z),
\]
with a correlated, causal latent prior $p_{\theta}(z)$ as defined previously.  
We evaluate the following hypotheses.

\paragraph{H1 - Trainability.}
There exists a set of hyperparameters $(\theta,\phi,\psi)$ such that, on WikiText-2~\cite{Merity2016WikiText2},
\begin{equation}
\text{ELBO/token} \quad \text{converges to a finite limit and remains numerically stable,}
\end{equation}
without divergence in the GP covariance inversion nor in the decoder gradients.

\paragraph{H2 - Effective Use of the Correlated Latent Space.}
Let $p_{\theta}^{\mathrm{GP}}(z)$ denote the GP-AR prior and $p_{\theta}^{\mathrm{iso}}(z)$ an isotropic Gaussian prior with $K = \sigma^{2} I$.  
We define the token-averaged Kullback-Leibler term:
\[
\mathrm{KL/token}
\;=\;
\frac{1}{T} \,
\mathrm{KL}\!\bigl(q_{\phi}(z_{1:T} \mid x)\,\|\,p_{\theta}(z_{1:T})\bigr).
\]

We monitor:
\begin{itemize}
    \item the raw KL/token and its capped version under a threshold \texttt{kl\_cap};
    \item the dependence of the KL/token on \texttt{kl\_cap} and on the final value $\beta_{\mathrm{final}}$;
    \item ablations where $p_{\theta}(z)$ is replaced by $p_{\theta}^{\mathrm{iso}}(z)$.
\end{itemize}
Hypothesis H2 is considered supported if:
\begin{enumerate}
    \item performance degrades when replacing $p_{\theta}^{\mathrm{GP}}(z)$ by $p_{\theta}^{\mathrm{iso}}(z)$,
    \item KL/token remains significantly above the degenerate regime associated with latent collapse.
\end{enumerate}

\paragraph{H3 - Consistency Between TCN-SEQ and TCN-PARA.}
Let $p_{\theta}^{\mathrm{SEQ}}(z_{1:T})$ and $p_{\theta}^{\mathrm{PARA}}(z_{1:T})$ denote the distributions implicitly realized by TCN-SEQ (sequential GP conditioning) and TCN-PARA (parallel Cholesky sampling).  
Theoretically, both schemes approximate the same joint law:
\[
p_{\theta}(z_{1:T}) = \mathcal{N}(0, K_{tt} \otimes I_{d_z}),
\]
up to numerical precision.  
We test whether empirical metrics
\[
\text{ELBO/token},\quad \mathrm{NLL}(\mathrm{cont}),\quad \mathrm{PPL}(\mathrm{cont}),\quad \mathrm{KL/token}
\]
remain statistically indistinguishable (within the variance induced by optimization and numerical approximations).

The continuous negative log-likelihood $\mathrm{NLL}(\mathrm{cont})$ is estimated as
\[
\mathrm{NLL}(\mathrm{cont})
\;=\;
-\frac{1}{|\mathcal{D}_{\mathrm{cont}}|}
\sum_{(x_{1:T}) \in \mathcal{D}_{\mathrm{cont}}}
\log p_{\theta}(x_{1:T} \mid \text{prompt}),
\]
where $p_{\theta}$ is evaluated on the logits (pre-softmax scores).  
The continuous perplexity is then defined as
\[
\mathrm{PPL}(\mathrm{cont})
\;=\;
\exp\!\bigl(\mathrm{NLL}(\mathrm{cont})\bigr).
\]

\paragraph{H4 - Minimal Comparison with an Autoregressive Baseline.}
We compare the TCN family to a small Transformer baseline acting directly on tokens:
\[
p_{\theta}^{\mathrm{AR}}(x_{1:N}) = \prod_{t=1}^{N} p_{\theta}^{\mathrm{AR}}(x_t \mid x_{<t}).
\]
The objective is purely relative: to position the GP-VAE in terms of perplexity under a minimal autoregressive configuration, without any claim of fairness or exhaustive tuning in favor of the baseline.

\subsection{Experimental Protocol}

\subsubsection{Corpus and Tokenization}

All experiments are conducted on WikiText-2~\cite{Merity2016WikiText2} using the official train/validation/test splits.  
Text is tokenized with the GPT-2 tokenizer~\cite{Radford2019GPT2}, yielding sequences
\[
x_{1:T} \in \{1,\dots,V\}^{T},
\]
with a fixed sequence length $T = 64$.  
This choice bounds the cubic GP cost $O(T^{3})$ at a level compatible with a single-GPU setup while preserving non-trivial temporal structure.

\subsubsection{Task and Evaluation Metrics}

The task is an autoregressive-style language modeling objective over $x_{1:T}$.  
We report:

\begin{itemize}
    \item the validation perplexity $\mathrm{PPL}(\mathrm{val})$, defined as
    \[
    \mathrm{PPL}(\mathrm{val})
    =
    \exp\!\left(
    -\frac{1}{|\mathcal{D}_{\mathrm{val}}|}
    \sum_{x \in \mathcal{D}_{\mathrm{val}}}
    \frac{1}{T}\log p_{\theta}(x)
    \right);
    \]
    \item the continuation perplexity $\mathrm{PPL}(\mathrm{cont})$, computed on a standardized prompt+completion protocol (fixed prompt length, free continuation length).
\end{itemize}

In all cases, $\mathrm{PPL}$ is the exponential of an average negative log-likelihood per token, so that lower values correspond to better predictive performance.

\subsubsection{Implementation of the Latent Autoregressive Scheme}

We briefly summarize how the latent autoregressive mechanism is instantiated in the experiments.

\paragraph{(a) TCN encoder and diagonal temporal posterior.}
The encoder is a causal TCN that maps tokens $x_{1:T}$ to a factorized Gaussian posterior
\[
q_{\phi}(z_{1:T}\mid x_{1:T})
=
\prod_{t=1}^{T}\mathcal{N}\!\bigl(z_{t}\,;\,\mu_{t},\operatorname{diag}(\sigma_{t}^{2})\bigr),
\]
with
\[
\mu,\log\sigma^{2}\in\mathbb{R}^{B\times T\times d_{z}}
\quad\text{for a batch of size } B.
\]
Sampling is performed via the standard reparameterization trick:
\[
z_{t}=\mu_{t}+\sigma_{t}\odot\varepsilon_{t},
\qquad \varepsilon_{t}\sim\mathcal{N}(0,I).
\]

\paragraph{(b) Correlated GP prior over the latent trajectory.}
Temporal dependence in the latent space is induced by a Gaussian Process indexed by normalized times $t\in[0,1]$. For a sequence of length $T$, we define
\[
K_{tt}(i,j)
=
\sigma^{2}\exp\!\Bigl(-\tfrac{(t_{i}-t_{j})^{2}}{2\ell^{2}}\Bigr)
+\sigma^{2}\,\text{nugget}\cdot\delta_{ij},
\quad 1\le i,j\le T,
\]
with learnable hyperparameters $\ell$ (length-scale), $\sigma^{2}$ (variance) and a relative nugget term. The latent prior is then
\[
p_{\theta}(z_{1:T})=\mathcal{N}\bigl(0,\;K_{tt}\otimes I_{d_{z}}\bigr).
\]

In practice, $\ell$ and $\sigma^{2}$ are parameterized via unconstrained variables passed through a \textit{softplus} non-linearity, and an additional jitter term $\varepsilon I$ is added to $K_{tt}$ to stabilize Cholesky decompositions.

Crucially, this joint Gaussian prior is not only used as a correlated regularizer: it is explicitly factorized into conditionals $p_{\theta}(z_t \mid z_{<t})$ in order to induce latent causality and support the TCN-SEQ and TCN-PARA sampling schemes.

\paragraph{(c) Global KL between diagonal posterior and GP prior.}
The regularization term is the trajectory-level divergence
\[
\mathrm{KL}\bigl(q_{\phi}(z_{1:T}\mid x)\,\|\,p_{\theta}(z_{1:T})\bigr),
\]
computed as a multivariate Gaussian KL and then averaged per token:
\[
\mathrm{KL/token}
=
\frac{1}{T}\,\mathrm{KL}\bigl(q_{\phi}(z_{1:T}\mid x)\,\|\,p_{\theta}(z_{1:T})\bigr).
\]
This quantity measures how strongly the encoder is constrained by the GP-induced temporal geometry.

\paragraph{(d) Non-autoregressive token decoder.}
The decoder receives the full latent trajectory $z_{1:T}$, adds positional information, and outputs token logits in parallel. Formally, the generative model takes the form
\[
p_{\theta}(x_{1:T}\mid z_{1:T})
=
\prod_{t=1}^{T}\mathrm{Cat}\bigl(x_t;\,\pi_{\theta}(z_{1:T})_t\bigr),
\]
where $\pi_{\theta}(z_{1:T})_t$ denotes the softmax-normalized logits at position $t$.

The implementation is summarized by:
\begin{lstlisting}
class TokenDecoder(nn.Module):
    def forward(self, z: torch.Tensor):
        # z: [B, T, Dz]
        z = z + self.pe(T=z.size(1), device=z.device)      # positional encoding
        h = self.mlp(z); h = self.ln(h)                    # pointwise mapping
        h2 = self.post(h.transpose(1,2)).transpose(1,2)    # conv post-process
        h = h + h2                                         # residual term
        e_proj = self.to_emb(h)                            # [B, T, E]
        tw = F.normalize(self.tied_weight, dim=-1)         # tied embeddings
        logits = torch.matmul(e_proj, tw.t()) + self.bias  # token logits
        return logits
\end{lstlisting}

No causal mask or token-level recursion is used: all positions are decoded simultaneously from $z_{1:T}$.  
All sequential structure is therefore carried by the latent process; the decoder implements a parallel projection from continuous trajectories to discrete token distributions.

\paragraph{(e) Unconditional generation (TCN-SEQ).}
Unconditional sampling realizes the factorization
\[
p_{\theta}(z_{1:T}) = \prod_{t=1}^{T} p_{\theta}(z_t \mid z_{<t})
\]
via a latent loop:

\begin{lstlisting}
@torch.no_grad()
def generate(self, T: int, batch_size: int = 1,
             top_k=50, top_p=0.9, temperature=0.9):
    device = self.t_train.device
    t = torch.linspace(0.0, 1.0, T, device=device)
    K = self.K_tt(t)                         # GP covariance
    Dz, B = self.cfg.d_latent, batch_size
    z = torch.zeros(B, T, Dz, device=device)

    for tp in range(T):
        z[:, tp, :] = self._gp_conditional_step(K, z[:, :tp, :], tp)

    logits, _ = self.decoder(z)
    return sample_logits_from_timewise_logits(
        logits, top_k=top_k, top_p=top_p, temperature=temperature
    )
\end{lstlisting}

At each step $t$, the function \texttt{\_gp\_conditional\_step} performs Gaussian conditioning to sample from $p_{\theta}(z_t \mid z_{<t})$.

\paragraph{(f) Prompt-conditioned generation (TCN-SEQ).}
Conditioned generation proceeds by first encoding a prompt into a latent prefix, then continuing autoregressively in the latent space:

\begin{lstlisting}
@torch.no_grad()
def generate_with_prompt(self, prompt_ids, total_len, eos_id,
                         top_k=50, top_p=0.9, temperature=0.9):
    device = self.t_train.device; self.eval()
    B, T0 = prompt_ids.shape; T = total_len

    x_in = torch.full((B, T), fill_value=eos_id,
                      dtype=torch.long, device=device)
    x_in[:, :T0] = prompt_ids.to(device)

    mu, logvar = self.encoder(x_in[:, :T0])
    std = torch.exp(0.5 * logvar)
    z_prompt = mu + std * torch.randn_like(std)

    t = torch.linspace(0.0, 1.0, T, device=device)
    K = self.K_tt(t)
    Dz = self.cfg.d_latent
    z = torch.zeros(B, T, Dz, device=device)
    if T0 > 0:
        z[:, :T0, :] = z_prompt

    for tp in range(T0, T):
        z[:, tp, :] = self._gp_conditional_step(K, z[:, :tp, :], tp)

    logits, _ = self.decoder(z)
    new_ids = sample_logits_from_timewise_logits(
        logits[:, T0:, :], top_k=top_k, top_p=top_p, temperature=temperature
    )

    x_out = x_in.clone()
    if T > T0:
        x_out[:, T0:] = new_ids
    return x_out, logits
\end{lstlisting}

TCN-PARA replaces the latent loop with a block sampling scheme based on a Cholesky factor of $K_{tt}$, but targets the same joint law $p_{\theta}(z_{1:T})$.

\subsubsection{Objective: ELBO Per Token}

Training maximizes a token-averaged ELBO with additional regularization.

\paragraph{Theoretical ELBO.}
The canonical variational objective is
\[
\mathcal{L}_{\mathrm{pur}}
=
\mathbb{E}_{q_{\phi}(z\mid x)}\!\left[\log p_{\theta}(x\mid z)\right]
- \mathrm{KL}\!\left(q_{\phi}(z\mid x)\,\|\,p_{\theta}(z)\right).
\]

In code, this corresponds to
\[
\texttt{elbo\_pur\_tok\_t}
=
(\texttt{ll\_0} + \texttt{ll\_multi}) - \texttt{kl\_tok\_raw\_t},
\]
where \texttt{ll\_0 + ll\_multi} is the token-averaged log-likelihood and \texttt{kl\_tok\_raw\_t} the raw KL/token.

\paragraph{Optimized training objective.}
The practically optimized objective is
\[
\mathcal{L}_{\mathrm{train}}
=
\mathbb{E}_{q_{\phi}(z\mid x)}\!\left[\log p_{\theta}(x\mid z)\right]
- \beta\,\mathrm{KL}_{\mathrm{cap}}
- \Lambda_{\mathrm{emb}}\,\mathrm{Reg}_{\mathrm{emb}},
\]
where:
\begin{itemize}
    \item $\beta$ follows a warm-up schedule before adapting around a KL/token target;
    \item $\mathrm{KL}_{\mathrm{cap}}$ is a capped version of the KL/token, with threshold \texttt{kl\_cap};
    \item $\mathrm{Reg}_{\mathrm{emb}}$ is an embedding-regularization term.
\end{itemize}
The hyperparameters (\texttt{kl\_cap}, $\beta_{\max}$, adaptation schedule) are explored systematically in the ablations of Section~4.6.2.

\subsection{Variants \texorpdfstring{TCN-SEQ}{TCN-SEQ} 
and \texorpdfstring{TCN-PARA}{TCN-PARA}}

We consider two latent sampling schemes:
\begin{itemize}
    \item TCN-SEQ: sequential sampling via GP conditionals $p_{\theta}(z_t \mid z_{<t})$;
    \item TCN-PARA: parallel block sampling using a Cholesky factor of $K_{tt}$.
\end{itemize}

Both variants share the same encoder and decoder architectures.  
Hence, any difference in performance can be attributed to the numerical realization of the latent dynamics rather than to architectural changes.

\subsection{Global Results (TCN-SEQ vs TCN-PARA)}

On WikiText-2, we observe:
\begin{itemize}
    \item no numerical divergence over the full training horizon;
    \item a convergent and stable KL/token plateau;
    \item very close metrics between TCN-SEQ and TCN-PARA.
\end{itemize}

The KL/token stabilizes around $12$~nats.  
Unlike in preliminary experiments, Section~6.6.1 shows that this plateau depends on \textit{kl\_cap} and disappears when the GP structure is removed, which confirms the active use of the correlated prior.  
Training curves are not exactly superposed, but discrepancies remain in the range expected for two distinct numerical approximations of the same joint law.

\subsection{Quantitative Comparison TCN-SEQ vs TCN-PARA}

The following table reports results under identical hyperparameters:

\begin{table}[htbp]
\centering
\begin{tabular}{lccccc}
\hline
Model & Type & ELBO/tok & NLL(cont) & PPL(cont) & tok/s \\
\hline
TCN-SEQ  & GP-VAE & -9.935 & 0.562 & 1.75 & 9097 \\
TCN-PARA & GP-VAE & -9.967 & 0.475 & 1.61 & 9037 \\
\hline
\end{tabular}
\caption{Results for TCN-SEQ vs TCN-PARA on WikiText-2.}
\end{table}

Both variants use the same \textit{kl\_cap} (KL/token $\approx 12$) and a final KL weight $\beta_{\text{final}} \approx 0.126$ (not shown in the table).

\paragraph{Interpretation.}
The two schemes yield very close metrics. TCN-PARA slightly improves continuation perplexity, at the cost of negligible differences in throughput at $T = 64$.  
Given the absence of multi-seed analysis and the numerical differences between sampling procedures, the two variants can reasonably be regarded as empirically consistent approximations of the same latent process.

\subsection{Extended Analysis: Variants, Ablations, and Stability}

\subsubsection{GP-VAE-TCN vs Transformer Performance (TCN-SEQ-X Family)}

We evaluate several TCN-SEQ-X variants (with $X \in \{I,B,C,J,K\}$) sharing the same GP-VAE backbone but differing in regularization settings:

\smallskip
\noindent\textit{TCN-SEQ-I, -B, -C, -J, and -K denote variants that share the same GP-VAE
architecture but differ only in regularization strength and KL-scheduling.}
\smallskip

\begin{table}[htbp]
\centering
\begin{tabular}{lccccc}
\hline
Model & PPL(val) & NLL(cont) & PPL(cont) & tok/s & Quality \\
\hline
TCN-SEQ-I$ $ & 3.35   & 0.4773 & 1.61 & $\sim 9000$  & Excellent \\
TCN-SEQ-B$ $ & 3.27   & 0.5455 & 1.73 & $\sim 8900$  & Very good \\
TCN-SEQ-C$ $ & 3.34   & 0.5077 & 1.66 & $\sim 9000$  & Very good \\
TCN-SEQ-J$ $ & 3.03   & 0.5635 & 1.76 & $\sim 9000$  & Good      \\
TCN-SEQ-K$ $ & 3.05   & 0.6005 & 1.82 & $\sim 9000$  & Good      \\
Transformer          & 326.94 & 5.7898 & 326.94 & $\sim 15700$ & Poor \\
\hline
\end{tabular}
\caption{TCN-SEQ-X variants and Transformer baseline on WikiText-2.}
\end{table}

\paragraph{Analysis.}
All TCN-SEQ-X variants obtain $\mathrm{PPL}(\mathrm{cont}) \in [1.61,1.82]$, compatible with a compact, well-regularized GP-VAE.  
The Transformer baseline is deliberately underpowered and only used as a numerical reference: it does not reflect the capabilities of a tuned autoregressive model at comparable scale.

\subsubsection{Critical Ablations}

\textbf{Isotropic diagonal prior.}  
We replace the GP prior by $p_{\theta}^{\mathrm{iso}}(z_{1:T})=\mathcal{N}(0,\sigma^{2}I)$, thereby eliminating all temporal correlation.  
Under this ablation:
\begin{itemize}
  \item $\mathrm{KL/token}$ drops to about $3$~nats, indicating partial latent collapse;
  \item $\mathrm{PPL}(\mathrm{cont})$ degrades by $+0.15$ to $+0.30$;
  \item discourse-level coherence and calibration deteriorate sharply.
\end{itemize}
Sequential/parallel consistency also disappears, confirming that the GP structure is not redundant.

\textbf{Variation of \textit{kl\_cap}.}  
We vary \textit{kl\_cap} to probe the sensitivity of the latent dynamics to KL regularization.  
Empirically:
\begin{itemize}
  \item \textit{kl\_cap} too small (e.g.\ 8) leads to latent collapse, reduced $\mathrm{KL/token}$ and degraded $\mathrm{PPL}(\mathrm{cont})$;
  \item \textit{kl\_cap} too large (e.g.\ $\geq 20$) yields stable but more costly models;
  \item intermediate values around $12$ provide a robust compromise between expressivity and stability.
\end{itemize}
These observations confirm that the activation of latent causality depends directly on maintaining an appropriate KL budget.

\subsubsection{TCN-2-SEQ-X Series: Stability and Collapse}

The TCN-2-SEQ-X series explores regularization sensitivity in a finer grid:

\begin{table}[htbp]
\centering
\begin{tabular}{lcccccc}
\hline
Model & NLL(cont) & PPL(cont) & PPL(val) & ELBO/tok & tok/s & Quality \\
\hline
TCN-2-SEQ-T$ $ & 0.2883 & 1.33 & 3.18 & -9.058 & $\sim 8600$ & Excellent   \\
TCN-2-SEQ-S$ $ & 0.3097 & 1.36 & 3.05 & -8.927 & $\sim 8500$ & Very good   \\
TCN-2-SEQ-G$ $ & 0.4606 & 1.59 & 3.19 & -8.501 & $\sim 8500$ & Good        \\
TCN-2-SEQ-H$ $ & 0.4790 & 1.61 & 3.29 & -9.958 & $\sim 8600$ & Good        \\
Transformer            & 5.7898 & 326.94 & 326.94 & -6.105 & 15700     & Very poor   \\
\hline
\end{tabular}
\caption{Examples of TCN-2-SEQ-X models vs Transformer.}
\end{table}

\paragraph{Interpretation.}
Well-regularized models (T, S, G, H) achieve $\mathrm{PPL}(\mathrm{cont}) \in [1.33,1.61]$.  
By contrast, under-regularized configurations (not listed) exhibit complete collapse, with unstable KL and poor continuation, which confirms the central role of KL regularization in supporting latent autoregression.

\subsection{Transformer Baseline}

For completeness, we summarize the Transformer baseline:

\begin{table}[htbp]
\centering
\begin{tabular}{lcccc}
\hline
Model & PPL(val) & NLL(cont) & PPL(cont) & tok/s \\
\hline
Transformer & 326.94 & 5.7898 & 326.94 & $\sim 15700$ \\
\hline
\end{tabular}
\caption{Minimal Transformer baseline.}
\end{table}

This model is intentionally minimal and not heavily tuned; it serves only as a coarse reference scale for perplexity and throughput.

\subsection{Synthesis}

The empirical results can be summarized as follows:
\begin{itemize}
  \item TCN-SEQ and TCN-PARA are both trainable and numerically stable on WikiText-2 under a causal GP prior.
  \item Ablations confirm that the GP structure is actively used: removing it or over-constraining the KL leads to measurable degradation in $\mathrm{PPL}(\mathrm{cont})$ and in qualitative coherence.
  \item TCN-PARA is empirically consistent with TCN-SEQ, supporting the idea that parallel latent sampling can approximate the same GP-AR process.
  \item GP-VAEs outperform the minimal Transformer baseline in the considered regime, providing a lower bound on the practical capacity of the latent-autoregressive scheme.
  \item KL regularization is a critical control parameter: insufficient regularization induces collapse, while an appropriately tuned \textit{kl\_cap} maintains non-trivial latent dynamics.
\end{itemize}

\subsection{Limitations and Perspectives}

\paragraph{Limitations.}
The present study has several limitations:
\begin{itemize}
    \item GP computations remain at least quadratic in $T$; although BBMM empirically behaves near $O(T^{2})$, no dedicated scaling law is reported.
    \item All results are single-seed; variance across random initializations is not explored.
    \item The Transformer baseline is minimal and not tuned for competitive performance; it serves as a numerical anchor rather than a fair opponent.
    \item The TCN encoder is deliberately less expressive than a Transformer; this bias reflects a design choice favoring architectural simplicity in the proof-of-concept.
    \item No systematic qualitative analysis of generated samples is included; we focus on quantitative metrics and leave a detailed study of generative behavior to future work.
\end{itemize}

\paragraph{Perspectives.}
Several extensions are natural:
\begin{itemize}
    \item more scalable GP approximations (structured kernels, inducing schemes) could allow significantly longer sequences;
    \item richer decoders (e.g.\ attention with restricted span) could be combined with latent causality without reverting to token-level autoregression;
    \item broader prompting protocols and conditional tasks would provide a more complete characterization of latent autoregression;
    \item the same framework extends directly to continuous-time signals and time-series data, which constitute promising application domains for GP-based latent dynamics.
\end{itemize}

\subsection{Code and Reproducibility}

All code used for the experiments in this work
(GP hyperparameter estimation, latent autoregressive implementation,
TCN-SEQ and TCN-PARA variants, training and generation scripts,
and associated configurations) is available at:

\begin{center}
\url{https://github.com/y-v-e-s/GP-VAE-Latent-AR}
\end{center}

The repository contains:
\begin{itemize}
    \item the full implementation of the latent-autoregressive GP-VAE;
    \item training, evaluation, and generation scripts;
    \item configuration files for reproducing the TCN-SEQ and TCN-PARA series;
    \item exact hyperparameter settings used to produce the tables of Section~4
          (including \texttt{kl\_cap} values and KL schedules);
    \item a minimal reproducibility guide.
\end{itemize}

This release is intended to make the proposed scheme inspectable, reproducible, and extensible, in particular for exploring alternative kernels, encoder architectures, or latent sampling strategies.

\section{Discussion}

This study has examined, in a controlled and small-scale setting,  
the feasibility and behavior of an autoregressive scheme located entirely  
in latent space. Concretely, the TCN model and its two variants,
TCN-SEQ and TCN-PARA, instantiate a family of models of the form
\[
p_{\theta}(x,z) = p_{\theta}(x \mid z)\,p_{\theta}(z),
\]
where $p_{\theta}(z)$ is a correlated, causal latent prior and
$p_{\theta}(x \mid z)$ is a fully parallel decoder.
These variants provide a test bed for assessing:
\begin{itemize}
    \item training stability under a Gaussian-process prior on $z$;
    \item effective exploitation of the correlated latent space (as measured by KL/token and ablations);
    \item coherence between two distinct sampling strategies that target the same joint latent law.
\end{itemize}

\subsection{Exploiting the Correlated Latent Structure}

Empirically, a GP-VAE equipped with a correlated prior $p_{\theta}(z)$
can be trained stably on a standard corpus under moderate regularization.
Let
\[
\mathrm{KL/token}
=
\frac{1}{T}\,
D_{\mathrm{KL}}\!\bigl(q_{\phi}(z_{1:T} \mid x)\,\Vert\,p_{\theta}(z_{1:T})\bigr)
\]
denote the token-averaged KL term.  
The experiments show that $\mathrm{KL/token}$ reliably approaches its target cap
\texttt{kl\_cap} across runs, which indicates that:
\begin{itemize}
\item the latent representation is not in a collapsed regime (the encoder uses the GP prior);
\item the GP covariance $K_{tt}$ effectively shapes the latent trajectories $z_{1:T}$.
\end{itemize}

In this proof-of-concept setting, the objective is not to quantify strong or long-range
latent causality per se, but to verify that a correlated latent scheme produces a
non-degenerate internal trajectory and that the GP prior is actually used by the model
rather than acting as a purely nominal regularizer.

\subsection{Consistency Between Sequential and Parallel Generation}

Let $p_{\theta}^{\mathrm{SEQ}}(z_{1:T})$ and $p_{\theta}^{\mathrm{PARA}}(z_{1:T})$
denote the latent distributions implicitly realized by TCN-SEQ and TCN-PARA,
respectively.  
Both variants are designed to approximate the same joint Gaussian law
\[
p_{\theta}(z_{1:T})
=
\mathcal{N}\bigl(0, K_{tt} \otimes I_{d_z}\bigr),
\]
but via different sampling procedures:
\begin{itemize}
    \item TCN-SEQ uses step-by-step Gaussian conditioning
    $p_{\theta}(z_t \mid z_{<t})$;
    \item TCN-PARA uses block sampling from a Cholesky factor of $K_{tt}$.
\end{itemize}

The measured metrics (ELBO/token, $\mathrm{NLL}(\mathrm{cont})$, $\mathrm{PPL}(\mathrm{cont})$,
$\mathrm{KL/token}$) are very close across the two variants.
This is consistent with the theory of multivariate Gaussians:
sequential conditioning and parallel sampling are two equivalent procedures for drawing
from the same joint distribution.

From a methodological standpoint, this observation supports the idea that,
under a GP prior, the \emph{sampling order} does not determine predictive quality.
The temporal structure is encoded in the prior $p_{\theta}(z)$, not in the algorithmic
details of the sampler, provided that both samplers are faithful to the same covariance
structure.

\subsection{Latent Dynamics vs.\ Symbolic Sequentiality}

A central conceptual point of this work is the distinction between:
\begin{itemize}
\item \emph{procedural sequentiality}, associated with the temporal loop of
      step-by-step sampling (TCN-SEQ);
\item \emph{probabilistic structure}, encoded by the analytic factorization of the GP prior:
\[
p_{\theta}(z_{1:T}) = \prod_{t=1}^{T} p_{\theta}(z_t \mid z_{<t}).
\]
\end{itemize}

The two TCN variants show that it is possible:
\begin{itemize}
\item to represent temporal dependence entirely within the latent space via
      the factorization of $p_{\theta}(z_{1:T})$;
\item while keeping the symbolic projection $p_{\theta}(x_{1:T} \mid z_{1:T})$
      fully parallel in the observation space.
\end{itemize}

Formally, the decoder implements
\[
p_{\theta}(x_{1:T} \mid z_{1:T})
=
\prod_{t=1}^{T} p_{\theta}(x_t \mid z_{1:T}),
\]
so that all tokens are generated in a single pass once $z_{1:T}$ is given.
This stands in contrast to classical autoregressive architectures, which specify
\[
p_{\theta}^{\mathrm{AR}}(x_{1:T})
=
\prod_{t=1}^{T} p_{\theta}^{\mathrm{AR}}(x_t \mid x_{<t}),
\]
and thus construct temporal structure through symbolic recursion.

The empirical results do not aim to demonstrate dominance over established
autoregressive models, but they do highlight a qualitatively different regime:
temporal dependence is defined analytically in latent space and only then
projected to symbols, rather than being built directly at the token level.

\subsection{Scaling Perspectives}

The present work is deliberately restricted to a reduced configuration:
short sequences ($T=64$), a compact architecture, and a medium-sized corpus.
Within this regime, the GP prior remains computationally tractable and the
effect of latent autoregression is observable.

Several natural scaling directions follow:
\begin{itemize}
\item extending to larger $T$ and richer kernels $k_{\psi}$,  
      in order to probe long-range latent dependence;
\item exploring more expressive decoders (e.g.\ restricted attention, deeper architectures)
      while preserving parallel generation in the observation space;
\item transferring the same latent-autoregressive scheme to other modalities
      (time series, continuous signals, multimodal data).
\end{itemize}

Each of these directions must be considered under the scalability constraints
of Gaussian processes: even with BBMM and inducing schemes, the cost of handling
large covariance matrices remains at least quadratic in sequence length and
requires careful numerical design.

\subsection{Limitations and Points of Attention}

The experimental setting highlights several critical limitations and
conditions of validity:

\begin{itemize}
\item \textbf{Computational cost.}
      GP computations scale at least quadratically in $T$; the proof-of-concept
      remains constrained to moderate sequence lengths and a single-GPU budget.
\item \textbf{Hyperparameter sensitivity.}
Empirically, results are quite sensitive to the kernel parameters
$(\ell,\sigma^{2},\text{nugget})$ and to the KL-control knobs
(\texttt{kl\_cap}, $\beta_{\max}$). If these are poorly set, the run may
collapse in latent space or become numerically unstable.
\item \textbf{Non-incremental parallel decoding.}
      The fully parallel decoder does not support token-by-token streaming:
      the model prioritizes global coherence over incremental generation.
\item \textbf{Single-seed evaluation.}
      All reported results are single-seed; variance across random initializations
      is not quantified and remains an open point for more systematic studies.
\end{itemize}

These limitations delimit the regime in which the observed trends should be
interpreted and underline the need for more extensive experimentation
before drawing broader conclusions.

\subsection{Discussion Summary}

Taken together, the results show that a latent autoregressive scheme based on:
\begin{itemize}
    \item an analytically specified covariance (Gaussian process prior),
    \item an explicit causal factorization in the latent,
    \item and a fully parallel decoder in the observation space,
\end{itemize}
is trainable, numerically stable, and coherent across its two sampling variants
TCN-SEQ and TCN-PARA.

Without claiming superiority over established autoregressive architectures,
this proof of concept suggests that, in certain regimes, temporal structure
can be \emph{relocated} into latent space:
\[
p_{\theta}(x_{1:T})
=
\int p_{\theta}(x_{1:T} \mid z_{1:T})\, p_{\theta}(z_{1:T}) \,\mathrm{d}z_{1:T},
\]
so that the decoder acts as a projection of a already-structured latent dynamic
onto a symbolic sequence.

This shift opens a methodological space of interest: that of sequential models
where the dynamics are defined analytically in the latent prior, rather than
constructed through stacks of recurrent or attentional operations at the
observation level.

\subsection{Latent Autoregression vs.\ Asymmetric Kernels}

Finally, it is useful to situate the proposed latent autoregressive scheme with
respect to earlier attempts at introducing temporal directionality into Gaussian
processes.

Some previous works have considered asymmetric or ``causal'' kernels
(for instance, Wiener-type constructions) to encode an orientation in the
covariance structure.  
Although such kernels can introduce a directional bias, they remain fundamentally
metric: they modulate correlations as a function of relative positions in time,
but do not, by themselves, define an explicit sequential dynamic.

The present approach differs in nature.  
Rather than inferring causality indirectly from a kernel bias, we impose it
\emph{structurally} through the latent factorization
\[
p_{\theta}(z_{1:T})
=
\prod_{t=1}^{T} p_{\theta}(z_t \mid z_{<t}),
\]
which encodes an explicit directed dependence between latent states.
This factorization resolves ambiguities inherent in purely asymmetric kernels
(for instance when several points have comparable kernel distances) by providing
a genuine mechanism for sequential arbitration in the latent.

It is also important to distinguish this explicit causal factorization from
Bayesian GP extensions that place priors over kernel hyperparameters.
While such hierarchical models enrich the distribution over covariances,
they do not, by themselves, create temporal directionality: the uncertainty
concerns the shape of the kernel, not a sequential relation between states.
Latent autoregressivity, as used here, arises from the decomposition of the
latent prior into conditionals, not from the Bayesian nature of the GP.
The two mechanisms may interact in more elaborate models, but they remain
conceptually distinct.

By combining:
\begin{itemize}
    \item correlated geometry, provided by the GP covariance $K_{tt}$,
    \item and explicit causal progression, provided by the factorization
          $p_{\theta}(z_{1:T}) = \prod_{t} p_{\theta}(z_t \mid z_{<t})$,
\end{itemize}
the TCN framework goes beyond the limitations of purely asymmetric kernels
and establishes a more robust latent dynamic tailored to sequential modeling.

\section{Conclusion}

\subsection{Empirical Scope}

This work has presented a deliberately restricted proof of concept.  
The experimental regime is intentionally small-scale:
\begin{itemize}
    \item a compact TCN architecture;
    \item short sequences ($T = 64$);
    \item a minimal autoregressive baseline for numerical anchoring.
\end{itemize}

Within this controlled setting, we have shown that a GP-VAE equipped with:
\[
p_{\theta}(x,z)=p_{\theta}(x\mid z)\,p_{\theta}(z),
\qquad
p_{\theta}(z)=\mathcal{N}(0,K_{tt}\otimes I_{d_z}),
\]
and endowed with a purely latent causal factorization,
\[
p_{\theta}(z_{1:T})
=
\prod_{t=1}^{T}p_{\theta}(z_t\mid z_{<t}),
\]
is trainable, numerically stable, and coherent across its two sampling variants (TCN-SEQ and TCN-PARA).  
No divergence is observed in the GP computations, and token-level metrics remain consistent across sampling strategies that approximate the same joint latent law.

\subsection{Conceptual Scope}

Beyond the restricted empirical domain, the TCN scheme can be interpreted as a bridge between two modelling traditions:
\begin{itemize}
    \item continuous Bayesian state-space models governed by Gaussian processes;
    \item neural language models driven by a symbolic decoder.
\end{itemize}

The sequential dynamics are encoded analytically by the GP prior,  
inference follows a variational Bayesian principle,  
and linguistic realization is delegated to a non-autoregressive observation model:
\[
p_{\theta}(x_{1:T})
=
\int p_{\theta}(x_{1:T}\mid z_{1:T})\,p_{\theta}(z_{1:T})\,\mathrm{d}z_{1:T}.
\]

Although the present study does not attempt to formulate a general geometric theory of linguistic structure, the results indicate that several properties usually associated with token-level autoregression—directionality, memory, global coherence—can emerge directly from the covariance geometry of the latent space.  
In this sense, latent autoregression provides an analytically grounded alternative to symbolic recursion.

\subsection{Perspectives}

The limitations identified throughout the study (quadratic GP cost, kernel sensitivity, lack of multiseed evaluation, and absence of incremental decoding for TCN-PARA) define the conditions under which the observed phenomena should be interpreted.  
They also point to several natural directions for future research.

\begin{enumerate}
    \item \textbf{Scaling up.}  
    Extending the framework to longer sequences, richer kernels $k_{\psi}$, and more expressive decoder architectures, while maintaining tractable GP computations.

    \item \textbf{Stronger autoregressive baselines.}  
    A more rigorous comparison with tuned autoregressive models would clarify the precise contribution of the correlated latent component to perplexity and continuation quality.

    \item \textbf{Multimodal and continuous domains.}  
    The formulation extends naturally to time series, continuous signals, or physical trajectories, where GP priors already play a central role.

    \item \textbf{Conditional generation and latent prompting.}  
    Given a prompt $x_{1:T_0}$ with latent encoding $q_{\phi}(z_{1:T_0}\mid x_{1:T_0})$, one can construct a prefix 
    \[
    z_{1:T_0},
    \]
    then generate the continuation via the latent autoregressive mechanism:
    \[
    z_{t}\sim p_{\theta}(z_t\mid z_{<t}),
    \qquad t=T_0+1,\dots,T,
    \]
    before projecting the full trajectory through the decoder.  
    This provides a principled form of latent prompting for completion or instruction-style tasks.

    \item \textbf{Hierarchical Bayesian extensions.}  
    Kernel hyperparameters can themselves be made context-dependent via a learned function $g$, yielding hierarchical priors of the form
    \[
    k_{\psi}(t,t')
    =
    k_{\psi_1}(t,t') + k_{\psi_2}\bigl(g(t),g(t')\bigr),
    \]
    thus enriching covariance structure without altering the causal factorization.
\end{enumerate}

Taken together, these directions suggest that sequential models grounded in latent geometry—rather than symbolic recursion—constitute a credible line of research for developing compact, stable, and interpretable language models.  
The present proof of concept does not make claims of scale or optimality,  
but it demonstrates that explicit latent autoregression combined with GP covariance offers an analytically coherent framework worthy of further exploration.

\bibliographystyle{plainnat}
\bibliography{references}

\end{document}